\documentclass[letterpaper, 10 pt, conference]{ieeeconf}

\IEEEoverridecommandlockouts

\usepackage{graphics} 
\usepackage{epsfig} 

\usepackage{times} 
\usepackage{amsmath}

\usepackage{amsthm}
\usepackage{amssymb}  
\usepackage{algorithm}
\usepackage{graphicx}
\usepackage{algpseudocode}
\usepackage{dirtytalk}
\usepackage{subcaption}
\usepackage{hyperref}
\usepackage{url}
\usepackage{float}
\usepackage{textcomp}

\usepackage{xcolor}
\renewcommand{\thesection}{\Roman{section}}
\allowdisplaybreaks 

\theoremstyle{plain}
\newtheorem{theorem}{Theorem}[section]
\newtheorem{corollary}[theorem]{Corollary}
\newtheorem{lemma}{Lemma}[section]

\title{Distributed Thompson sampling under constrained communication}
\author{Saba Zerefa, Zhaolin Ren, Haitong Ma, Na Li
\thanks{This work was supported by NSF AI institute 2112085, NSF ASCENT, 2328241, and NSF CNS: 2003111.}
\thanks{S. Zerefa, Z. Ren, H. Ma, and N. Li are affiliated with Harvard School of Engineering and Applied Sciences. Emails: szerefa@g.harvard.edu, zhaolinren@g.harvard.edu, haitongma@g.harvard.edu, and nali@seas.harvard.edu}
}

\begin{document}
\thispagestyle{plain}
\pagestyle{plain}
\maketitle
\renewcommand{\thetheorem}{\arabic{section}.\arabic{theorem}} 

\renewenvironment{proof}{{\itshape Proof.}}{\hfill $\qed$}

\begin{abstract}
In Bayesian optimization, a black-box function is maximized via the use of a surrogate model. We apply distributed Thompson sampling, using a Gaussian process as a surrogate model, to approach the multi-agent Bayesian optimization problem. In our distributed Thompson sampling implementation, each agent receives sampled points from neighbors, where the communication network is encoded in a graph; each agent utilizes their own Gaussian process to model the objective function. We demonstrate theoretical bounds on Bayesian average regret and Bayesian simple regret, where the bound depends on the structure of the communication graph. Unlike in batch Bayesian optimization, this bound is applicable in cases where the communication graph amongst agents is constrained. When compared to sequential single-agent Thompson sampling, our bound guarantees faster convergence with respect to time as long as the communication graph is connected. We confirm the efficacy of our algorithm with numerical simulations on traditional optimization test functions, demonstrating the significance of graph connectivity on improving regret convergence. 
\end{abstract}


\section{Introduction}
Black-box stochastic optimization involves solving problems where the objective function is not explicitly known and can only be accessed through noisy evaluations \cite{stochOpt}. These challenges frequently arise in domains where the evaluation process is costly and uncertain, such as hyperparameter tuning in machine learning \cite{hyperopt}, \cite{smac3}, simulation-based optimization \cite{simOpt}, and experimental design \cite{chem}. A variety of methods have been developed to tackle these problems, including evolutionary algorithms \cite{ea}, particle swarm optimization \cite{PSO}, and finite-difference methods \cite{zeroB}. Among these, Bayesian optimization (BO)  \cite{kushner, mockus}, has emerged as a particularly powerful framework. In contrast to the aforementioned black-box stochastic optimization algorithms which tend to be model-free, by leveraging a probabilistic surrogate model, often a Gaussian process (GP) \cite{gpbook}, BO not only handles the stochastic nature of the evaluations but also balances exploration and exploitation given an appropriately chosen surrogate-based sampling strategy. This data-efficient approach makes BO especially well-suited for optimizing expensive, noisy black-box functions~\cite{frazier}. Moreover, theoretically, BO is also known to satisfy finite-time convergence guarantees to global optima (which we note comes at the cost of a dependence on a term that depends on the complexity of the kernel used to model the underlying function)~\cite{frazier}. To the best of our knowledge, apart from BO, finite-time convergence rates in stochastic optimization are only available for finite-difference type methods. However, due to the local nature of finite-difference type methods, the corresponding finite-time convergence rates for such methods only guarantee convergence to stationary points~\cite{zeroB,zeroA}, in contrast to the convergence to global optima achieved by BO algorithms.

In BO, the generation of new sampling points is based on the current surrogate model. A good sampling strategy should balance exploration and exploitation of the current surrogate, which is key for efficient optimization. Common sampling strategies include acquisition function-based approaches such as expected improvement (EI) \cite{mockus} and BO-upper confidence bound (UCB) \cite{kakade}. Another popular sampling strategy is Thompson sampling, where the next query point is selected as the optimizer of a random function realization sampled from the current posterior~\cite{TS,TSTut}. To evaluate algorithm performance, \textit{regret} is studied, which quantifies the gap between the performance of sampled points and the global optimum \cite{regret}. Types of regret include simple regret, which measures the gap between the optimal value and the performance of the best queried point \cite{simpleregret}, and cumulative regret, which measures the sum of the gaps between the optimal value and the performance of each queried point~\cite{regret2}, \cite{agrawal}.

We are interested in multi-agent BO, where multiple agents can sample the objective function at a single timestep. Much of existing multi-agent BO literature studies batch BO, in which a central coordinator has access to each agent’s acquired information \cite{haitong}, \cite{parallelizedbo}. It then computes the sampling decisions for all agents, and communicates these decisions to each agent. These decisions are disseminated in batches, allowing multiple agents to simultaneously sample points, parallelizing the optimization process \cite{moreBatch}, \cite{ren2024minimizing}. 

Centralized approaches are inapplicable in distributed cases, in which there is no centralized coordinator and each agent must possess a local instance of the optimization algorithm \cite{distributed}. Additionally, they often do not scale well, as they require a central coordinator to manage the processing of all agents' data. Distributed networks are prevalent in real-world applications, such as in multi-robot source seeking and sensor networks \cite{haitong}, \cite{ss}. It may not be the case that all agents have access to all prior sampled points as in the batch setting - communication may be constrained, where some agents are only able to communicate with specific other agents \cite{naomi}. These constraints may be due to limited communication capacity or computational capacity of the agents, or due to physical proximity constraints. Prior literature providing theoretical guarantees for distributed Bayesian optimization require fully connected communication graphs, even in asynchronous cases \cite{parallelizedbo}, \cite{fullydbo}, and thus are inapplicable in settings with constrained communication. In this work, we study the distributed setting with constrained communication, in which at each round, agents send their sampled points to their neighbors and receive points sampled by their neighbors.

\textbf{Our contribution}: We propose a distributed Thompson sampling algorithm for the multi-agent Bayesian optimization problem under constrained communication. In the algorithm, each agent uses their own GP model to pick sampling points via Thompson sampling, and shares the queried points with its neighbors. We provide provable guarantees for the proposed distributed BO. In Theorem \ref{theorem:rab}, we establish a Bayesian average regret bound of $\tilde{O}\left(\frac{\sqrt{\theta(G)}}{\sqrt{Mt}}\right)$, where $M$ is the number of agents, $t$ is the number of optimization rounds, and $\theta(G)$ represents the clique cover number of $G$, i.e. $\theta(G)$ is the smallest number $L$ such that the graph $G$ can be decomposed into $L$ disjoint complete subgraphs. This implies then that the average regret bound is smaller for graphs with higher connectivity, which can be decomposed into a few large disjoint complete subgraphs. We also characterize Bayesian simple regret, demonstrating a bound of $\tilde{O}\left(\sqrt{\frac{1}{t |V_{max}|}} \right),$ where $|V_{max}|$ is the size of the largest complete subgraph of the communication network $G$. We note that this convergence speed is $O(\sqrt{|V_{max}|})$ times better than the best known simple regret rate for sequential single-agent BO, which is $\tilde{O}\left(\sqrt{\frac{1}{t}} \right)$~\cite{kakade}. We numerically test our algorithm on two standard optimization test functions \cite{optTest} with Erdős-Rényi graphs, demonstrating the efficiency of our algorithm. We find that lower regret is achieved with graphs of higher connectivity, supporting our theoretical results.


\section{Problem Formulation and Preliminaries}
\subsection{Problem Formulation}
For a compact set $\mathcal{X}\subset \mathbb{R}^d$, consider an unknown continuous function $f:\mathcal{X}\to\mathbb{R},$ with optimizer $x^*$. The goal is to find the maximum of this function, where we are only able to sample $f$ through expensive and noisy evaluations. We assume any of $M$ agents can query $f$ at any point and receive a noisy value $y=f(x)+\epsilon,$ with $\epsilon \sim \mathcal{N}(0,\sigma_\epsilon^2).$ Agents query $f$ throughout a total of $T$ iterations. For agent $i\in\{1,\ldots,M\}$ and iteration $t\in\{1,\ldots,T\},$ $x_{t,i}$ is the query point, and $y_{t,i}$ is the corresponding evaluation. Define $X_{t,i}=\{x_{1,i},\ldots,x_{t,i}\},$ $Y_{t,i}=\{y_{1,i},\ldots,y_{t,i}\}$ to be the queries and evaluations made by agent $i$ up to time $t.$ The communication network of $M$ agents is described by graph $G=(V,E),$ where $|V|=M,$ and $E\subset\{\{i,j\}:i,j\in V, i\neq j\}.$ An unordered pair $\{i,j\}\in E$ if agents $i$ and $j$ are able to communicate with each other. Additionally, we denote the set of neighbors of agent $i$ as $N(i)=\{j: \{i,j\}\in E\}.$ The data accessible to agent $i$ at time $t$ is $D_{t,i}=\{(x_{\tau,j},y_{\tau,j})\}_{j\in N(i)\cup i, \tau<t}.$ The set $D_{t,i}$ contains all sampled points up to time $t$ by agent $i$ and its neighbors. We do not make any assumptions regarding the structure of the communication network. The graph may even be unconnected. Our analysis will show how the graph structure affects the algorithm's performance. 

\subsection{Gaussian Process}
We use a Gaussian process (GP) to model our unknown objective function $f$ in our BO setting. Recall the unknown continuous objective function $f:\mathcal{X}\to\mathbb{R}.$ 
Let $\mathbf{X}_{D_t}=\{x_1,x_2,\ldots,x_t\},$ where $x_j$ is the $j$th evaluated point, and let $k:\mathcal{X}^2\to\mathbb{R}$ be a kernel function. Define
\begin{align*}\mu_{D_t}(x)&=\mathbf{k}_t(x)^\intercal(\mathbf{K}_{D_t}+\sigma^2_n\mathbf{I})^{-1}\mathbf{y}_{D_t}\\
k_{D_t}(x,x')&=k(x,x')-\mathbf{k}_{D_t}(x)^\intercal(\mathbf{K}_{D_t}+\sigma_n^2\mathbf{I})^{-1}\mathbf{k}_{D_t}(x'),\end{align*}
where $\mathbf{K}_{D_t}:=[k(x',x'')]_{x',x''\in\mathbf{X}_{D_t}},$ $\mathbf{k}_{D_t}(x):=[k(x',x)]_{x'\in\mathbf{X}_{D_t}}$ and $\mathbf{y}_{D_t}=\{f(x')+\epsilon'\}_{x'\in\mathbf{X}_{D_t}},$ where $\epsilon'\sim\mathcal{N}(0,\sigma_\epsilon^2).$ Thus we can define our GP, in which we denote $f|\mathcal{F}_{D_t}\sim GP (\mu_{D_t}(x),k_{D_t}(x,x')).$ Note that due to the nature of the $GP$, it is the case that for any $x\in\mathcal{X},$ $f(x)|\mathcal{F}_{D_t}\sim N(\mu_{D_t}(x),\sigma_{D_t}^2(x)),$ where $\sigma_{D_t}^2(x)=k_{D_t}(x,x)$ \cite{gpbook}. Furthermore, recall the distributed multi-agent setting, where each of $M$ agents have access to queried points in set $D_{t,i},$ where $D_{t,i}$ and $D_{t,j}$ may not be equal for distinct agents $i$ and $j.$ In our distributed setting, each agent $i$ has a unique GP model of $f$ at time $t$, $\mathcal{GP}_{t,i}$, since the data $D_{t,i}$ available to each agent $i$ is different. Thus we denote $f \mid \mathcal{F}_{D_{t,i}} \sim \mathcal{GP}_{t,i}(\mu_{D_{t,i}}(x),k_{D_{t,i}}(x,x'))$. In the GP framework, $x^*,$ the optimizer of $f,$ is treated as a random variable. As a result, $x^*$ has a posterior distribution $x^*|D_t,$ which is the optimal value of the GP $f|D_t.$ This structure is leveraged in our algorithm.

The kernel function $k(\cdot,\cdot)$ can be selected to reflect prior beliefs about the objective function $f,$ such as function smoothness \cite{gpbook}. Common selections of kernel functions include Linear, Squared Exponential, and Matérn kernels, the latter of which was used in the numerical implementations of our algorithm. Note that the GP problem structure does not make any assumptions regarding function convexity, and that for common kernels, convexity is not reflected by kernel selection. 

\subsection{Regret}
Our metric for algorithm performance is regret, which is an assessment of the quality of sampled points. We consider average regret, which quantifies the difference between the optimal value of the function and the queried value for each sampled point. In average regret, this difference is accumulated across all agents and timesteps, and then averaged by the amount of sampled points. To account for randomness of $f$ in our regret expression, we take expectation of average regret to yield the following expression, which we call Bayesian average regret:

\begin{equation}\label{eq:ab}
R_{AB}(t)=\frac{1}{tM}\sum_{\tau=1}^t\sum_{i=1}^M\mathbb{E}[f(x^*)- f(x_{\tau,i})]
\end{equation}

We also consider the simple regret, which is the difference between the optimal value of the function and the best value achieved amongst the previous queried points. This definition of regret is useful because optimization settings focus on locating the extrema of a function, and the simple regret tracks the smallest gap between the value at a sampled point and the optimal value. We take the expectation of simple regret to yield the following expression, which is called Bayesian simple regret:

\begin{equation}\label{eq:sb}
R_{SB}(t)=\min_{i\in\{1,2,\ldots,M\},\tau \in \{1,2,\ldots,t\}}\mathbb{E}[f(x^*)- f(x_{\tau,i})]
\end{equation}

In our theoretical analysis, we provide bounds on Bayesian average regret and Bayesian simple regret.

\subsection{Thompson Sampling}
Thompson sampling is an algorithm for sequential decision making that can be utilized in this context for determining the next point of the objective function to query \cite{TS}. When using Thompson sampling in our Bayesian optimization framework, an acquisition function is sampled from the posterior distribution of the Gaussian process. The maximizer of this function is the next query point at which the black-box objective function is sampled. The Gaussian process is then updated with new information from this sample, and the process repeats for the duration of the experiment.

In sequential single-agent Thompson sampling, each subsequent query point is determined based on a single model updated on all prior sampled points. Alternatively, in batch Thompson sampling, multiple query points are determined as a set at each round, and the objective function is sampled in parallel \cite{parallelizedbo}, \cite{ren2024minimizing}. Batch Thompson sampling is advantageous in systems capable of parallelizing, e.g. multi-agent systems, because it allows for convergence in fewer number of rounds than sequential single-agent Thompson sampling. 

Batch Thompson sampling is centralized, with all agents having access to the same information. However, this may not be realistic in real-world situations, where communication between agents may be constrained due to bandwidth limitations, computational constrictions, or privacy concerns. In these cases, agents may only have access to the sampled points by few other agents, and thus datasets available to distinct agents may differ. We propose a distributed Thompson sampling algorithm for this constrained communication case, and provide theoretical guarantees for the algorithm. 

\section{Algorithm: Distributed Thompson Sampling}
In our implementation of distributed Thompson sampling, each of $M$ agents have distinct Gaussian processes $\mathcal{GP}_i$ for modeling the objective function. At each time step $t,$ all agents update their GPs with the data history available to them. The agent then queries the objective function at $x_{t,i},$ which is the maximizer of the acquisition function sampled from the posterior GP, $\hat{f}_{t,i}\sim \mathcal{GP}_{t,i}.$ Each agent then communicates its sampled point to its neighbors, receives the points sampled by their neighbors, and updates their data history accordingly. The collection of data received by neighbors of agent $i$ at time $t$ is denoted as $C_{t,i}=\{(x_{t,j},y_{t,j})\}_{j\in N(i)}.$ Our method is shown in Algorithm \ref{alg:TSReg}. 
We stress that while we do assume a synchronous global clock, there exists no centralized coordinator in our algorithm that coordinates the queries of the different agents.
\begin{algorithm}[H]
\caption{Distributed Thompson Sampling}
\label{alg:TSReg}
\begin{algorithmic}[1] 
 \item Place GP prior on $f$
\For{i$= 1,\ldots,M$} 
    \State Initial data $D_{1,i}$
    \State $\mathcal{GP}_{0,i}\leftarrow GP$
\EndFor
    \For{$t= 1,\ldots,T$} 
        \For{i$= 1,\ldots,M$} 
            \State Update posterior $\mathcal{GP}_{t,i}$ conditioned on $D_{t,i}$ 
            \State Sample $\hat{f}_{t,i}\sim \mathcal{GP}_{t,i}$
            \State Choose next query point 
            \State \hspace{1cm}$x_{t,i}\leftarrow \arg\max_x \hat{f}_{t,i}(x)$
            \State Observe $y_{t,i}$
            \State Broadcast $(x_{t,i},y_{t,i})$  to neighbors $N(i)$;
            \State Collect evaluations $C_{t,i}$ from neighbors $N(i)$ 
            \State Update data history $D_{t+1,i}\leftarrow D_{t,i}\cup C_{t,i}\cup \{(x_{t,i},y_{t,i})\}$
        \EndFor
    \EndFor
\end{algorithmic}
\end{algorithm}

In step 11, we select the next sampling point of the objective function by finding the argmax of a function drawn from the posterior distribution of the GP. In our numerical implementation, we did so using gridsearch, but such an approach is computationally expensive for higher dimensional search spaces. Efficient computation of the argmax for Thompson sampling in high dimensional spaces is an active area of research, and a direction for future work.

\subsection{Theoretical Result}

We analyze the performance of the distributed Thompson sampling algorithm on the Bayesian average regret and Bayesian simple regret metrics. Our regret bound depends on the number of timesteps $T$ and the structure of the agent communication graph G. As in prior work, we utilize notions from information theory in our regret bound \cite{russovanroy}.

Our regret bound involves the Maximum Information Gain (MIG), which is a constant that captures the complexity of the objective function. MIG has been shown to be bounded for several kernel functions commonly used with GPs, including Squared Exponential and Matérn kernels, the latter of which was used in our numerical implementation \cite{kakade}. 

Let $D=\{x_1,\ldots,x_t\}\subset \mathcal{X},$ and define $y_D=\{(x,f(x)+\epsilon):x\in D\}.$ The MIG is denoted as 

\begin{equation}\label{eq:MIG}\Psi_t=\max_{D\subset \mathcal{X},|D|=t} I(f;y_D),\end{equation} 

\noindent where $I$ is the Shannon Mutual Information. The MIG $\Psi_t$ represents the largest mutual information gain from $f$ by sampling $t$ points. Additionally, for any positive integer $n$, we define the constant $\xi_n,$ which bounds the information gain of the current round of evaluations \cite{parallelizedbo}. Suppose $|D|=t$ points were already sampled, and $i$ points are being queried in the current round of evaluations, with $i<n;$ denote these points in set $A,$ where $A\subset \mathcal{X},$ and $y_{A}=\{(x,f(x)+\epsilon):x\in A\}.$ Then for $i\geq 1,$ $\xi_n$ satisfies 

\begin{equation}\label{eq:xi}\max_{A\subset\mathcal{X},|A|<n} I(f;y_{A}|y_D)\leq \frac12\log(\xi_n).\end{equation}
We next provide a bound for Bayesian average regret for an $M$ agent system with communication graph $G.$
\begin{theorem}
\label{theorem:rab}
Suppose $k(x,x')\leq 1$ for all $x,x'.$ Let $\{G_k\}_{k\in\{1,\ldots,n\}}$ be a collection of $n$ disjoint complete subgraphs of communication graph $G=(V,E)$, where $G_k=(V_k,E_k),$ and $\cup_{k\in\{1,\ldots,n\}}V_k=V.$ Then the Bayesian average regret after $t$ timesteps satisfies $R_{AB}(t)\leq \frac{1}{M}\sum_{k=1}^n |V_k| (\frac{C_1}{t|V_k|}+\sqrt{\frac{C_2\xi_{|V_k|}\beta_t\Psi_{t|V_k|}}{t|V_k|}}),$ where $\beta_t=2\log(t^2M|\mathcal{X}|),$ $C_1=\frac{\sqrt{2}\pi^{3/2}}{12},$ and $C_2=\frac{2}{\log(1+\sigma_\epsilon^{-2})}.$
\end{theorem}

\begin{proof}
The structure of our proof follows techniques from Kandasamy et al. \cite{parallelizedbo}. We aim to provide a bound on Bayesian average regret. Our proof begins by noting that we can develop an expression for Bayesian average regret by considering the Bayesian average regret of specific subsets of agents. We then decompose this into three sums, each of which utilize a confidence function $U_{t,i}(\cdot).$ We bound each of these sums using notions from information theory, allowing us to use information gain constants introduced in Equations \ref{eq:MIG} and \ref{eq:xi} to analyze the efficacy of the sampling process.

We bound the Bayesian average regret affiliated with agents in communication graph $G$ by bounding the Bayesian average regret within each complete subgraph of graph $G.$
Let $G=(V,E)$ be the communication graph for the $M$ agents. We can construct a collection of $n$ disjoint complete subgraphs $\{G_k\}_{k\in\{1,\ldots,n\}},$ where each $G_k=(V_k,E_k)$ is a subgraph of $G,$ with $\cup_{k\in\{1,2,\ldots,n\}} V_k=V.$ 

Recall from Equation \ref{eq:ab} that Bayesian average regret $R_{AB}(t)=\frac{1}{tM}\sum_{\tau=1}^t\sum_{i=1}^M\mathbb{E}[f(x^*)- f(x_{\tau,i})].$ We also introduce $R_{AB,k}(t)= \frac{1}{t|V_k|}\sum_{\tau=1}^t \sum_{i\in V_k}\mathbb{E}[f(x^*)- f(x_{\tau,i})],$ which is the Bayesian average regret affiliated with agents in $V_k.$ Recalling the partition of the vertex set $V$ into $\{V_k\}_{k\in\{1,\ldots,n\}},$ we may rewrite Bayesian average regret as follows:

\begin{align*}
R_{AB}(t)&=\frac{1}{tM}\sum_{\tau=1}^t\sum_{i=1}^M\mathbb{E}[f(x^*)- f(x_{\tau,i})]\\
&=\frac{1}{tM}\sum_{\tau=1}^t\sum_{k=1}^n\sum_{i\in V_k}\mathbb{E}[f(x^*)- f(x_{\tau,i})]\\
&=\frac{1}{tM}\sum_{k=1}^n\sum_{\tau=1}^t\sum_{i\in V_k}\mathbb{E}[f(x^*)- f(x_{\tau,i})]\\
&=\frac{1}{M}\sum_{k=1}^n |V_k|R_{AB,k}(t)
\end{align*}

Thus, it suffices to focus on bounding $R_{AB,k}(t).$ Define $U_{t,i}(x)=\mu_{D_{t,i}}(x)+\beta_t^{1/2}\sigma_{D_{t,i}}(x).$ To upper bound $R_{AB,k}(t),$ we can decompose the sum $\sum_{\tau=1}^t \sum_{i\in V_k}\mathbb{E}[f(x^*)- f(x_{\tau,i})]$ as follows:

\begin{align*}
&\sum_{\tau=1}^t \sum_{i\in V_k}\mathbb{E}[f(x^*)- f(x_{\tau,i})]\\
&=\sum_{\tau=1}^t\sum_{i\in V_k}\mathbb{E}[f(x^*)-U_{\tau,i}(x^*)
+U_{\tau,i}(x^*)-U_{\tau,i}(x_{\tau,i})\\
&\hspace{2cm}+U_{\tau,i}(x_{\tau,i})
-f(x_{\tau,i})]\\
&=\sum_{\tau=1}^t\sum_{i\in V_k} \underbrace{\mathbb{E} [f(x^*)-U_{\tau,i}(x^*)]}_{S1} +\underbrace{\mathbb{E} [U_{\tau,i}(x^*)-U_{\tau,i}(x_{\tau,i})]}_{S2}\\
&\hspace{2cm}+\underbrace{\mathbb{E} [U_{\tau,i}(x_{\tau,i})
-f(x_{\tau,i})]}_{S3} 
\end{align*}

We will now bound each of these sums. 

\textbf{S1}. Let's begin by upper bounding the sum \hspace{2.5cm}$S1=\sum_{\tau=1}^t\sum_{i\in V_k} \mathbb{E} [f(x^*)-U_{\tau,i}(x^*)].$

\begin{align}
S1&=\sum_{\tau=1}^t\!\sum_{i\in V_k} \mathbb{E} [f(x^*)-U_{\tau,i}(x^*)] \\
&\leq \sum_{\tau=1}^t\sum_{i\in V_k} \mathbb{E}\!\!\left[\mathbb{E} [\mathbb{I}\{f(x^*)\!>\!U_{\tau,i}(x^*)\} \!\!\left(f(x^*)\!-\!U_{\tau,i}(x^*) \right)\!\mid\! \mathcal{F}_{D_{t,i}}]\right] \\
&\leq  \sum_{\tau=1}^t\sum_{i\in V_k} \sum_{x\in\mathcal{X}} \frac{e^{-\beta_\tau/2}}{\sqrt{2\pi}} \leq \sum_{\tau=1}^t\frac{1}{\tau^2\sqrt{2\pi}} \leq \frac{\sqrt{2}\pi^{3/2}}{12} 
\end{align}

Line (6) upper bounds $S1$ by positive terms by making use of the indicator function $\mathbb{I}\{f(x^*)>U_{\tau,i}(x^*)\},$  which takes the value of $1$ when the condition $f(x^*)>U_{\tau,i}(x^*)$ is satisfied and $0$ otherwise. Line (7) utilizes Lemma \ref{lem:ething} in Appendix \ref{appendix:additional_analysis} to bound the expectation of positive terms in a normal distribution. Additionally, Line (7) results from substituting for $\beta_\tau,$ and the fact that $\sum_{j=1}^\infty \frac{1}{j^2}=\frac{\pi^2}{6}.$ Therefore, we have established an upper bound for $S1$.

\textbf{S2}. We evaluate the expression $S2=\sum_{\tau=1}^t\sum_{i\in V_k}\mathbb{E} [U_{\tau,i}(x^*)-U_{\tau,i}(x_{\tau,i})].$ First let's focus on the interior of the summation, $\mathbb{E} [U_{\tau,i}(x^*)-U_{\tau,i}(x_{\tau,i})].$ We will proceed to show that this expression evaluates to $0.$

By the law of total expectation, $\mathbb{E} [U_{\tau,i}(x^*)-U_{\tau,i}(x_{\tau,i})]=\mathbb{E}[\mathbb{E} [U_{\tau,i}(x^*)-U_{\tau,i}(x_{\tau,i})]|D_{t,i}]$. Because $x_{t,i}$ is sampled from the posterior distribution of $x^*|D_{t,i},$ we must have that $x_{t,i}$ and $x^*$ have the same distribution after conditioning on the acquired data, and thus $x_{t,i}|D_{t,i}\sim x^*|D_{t,i}.$ We also notice that $U_{t,i}$ is deterministic when conditioned on $D_{t,i}.$ Thus we have that $\mathbb{E}[\mathbb{E} [U_{\tau,i}(x^*)-U_{\tau,i}(x_{\tau,i})]|D_{t,i}]]=0,$ and consequently $S2=0.$

\textbf{S3}. Lastly, we bound $S3=\sum_{\tau=1}^t\sum_{i\in V_k} \mathbb{E} [U_{\tau,i}(x_{\tau,i})
\!-\!f(x_{\tau,i})].$ For the evaluation of this sum, we will introduce some additional notation. Denote $\bar{D}_t=\{(x_{\tau,j},y_{\tau,j})\}_{j\in V_k,\tau<t}.$ We can think of $\bar{D}_t$ as representing the data acquired by round $t$ in a batch setting exclusively by agents in $V_k$. Additionally, define
$\sigma_{t,i}(x) := \sigma(x) \mid \bar{D}_t \cup \left\{(x_{t,j},y_{t,j}) \right\}_{j \in V_k, j < i}.$ We pick to define $\sigma_{t,i} (\cdot)$ in this way to impose an ordering of the acquired data; in this sense, $\sigma_{t,i}(\cdot)$ depends on $(t-1)|V_k|+i-1$ previously sampled points. With this additional notation, we are well-equipped to upper bound $S3.$

\begin{align}
S3&=\sum_{\tau=1}^t\sum_{i\in V_k}\mathbb{E} [U_{\tau,i}(x_{\tau,i})-f(x_{\tau,i})] \\
&=\sum_{\tau=1}^t\sum_{i\in V_k}\mathbb{E} [\mathbb{E}[\mu_{D_{\tau,i}}(x_{\tau,i})+\beta_{\tau}^{1/2}\sigma_{D_{\tau,i}}(x_{\tau,i})\nonumber\\
&\hspace{1.75cm}-f(x_{\tau,i})]|D_{\tau,i}]]\\
&=\sum_{\tau=1}^t\sum_{i\in V_k}\mathbb{E} [\beta_{\tau}^{1/2}\sigma_{D_{\tau,i}}(x_{\tau,i})]\\
&\leq \beta_{t}^{1/2}\sum_{\tau=1}^t\sum_{i\in V_k} \mathbb{E}[\sigma_{D_{\tau,i}}(x_{\tau,i})]\\
&\leq \beta_{t}^{1/2}\sum_{\tau=1}^t\sum_{i\in V_k}\mathbb{E}[\sigma_{\bar{D}_{\tau}}(x_{\tau,i})] \\
&\leq \beta_{t}^{1/2}\mathbb{E}\left[\sum_{\tau=1}^t\sum_{i\in V_k}\sigma_{\tau,i}(x_{\tau,i})\exp(I(f;\{y_{\tau,j}\}_{j<i}|y_{\bar{D}_\tau})\right]\\
&\leq \beta_{t}^{1/2}\mathbb{E}\left[\sum_{\tau=1}^t\sum_{i\in V_k}\sigma_{\tau,i}(x_{\tau,i})\xi_i^{1/2}\right] \\
&\leq \beta_{t}^{1/2}\xi_{|V_k|}^{1/2}\mathbb{E}\left[\sum_{\tau=1}^t\sum_{i\in V_k}\sigma_{\tau,i}(x_{\tau,i})\right]\\
&\leq  \beta_{t}^{1/2}\xi_{|V_k|}^{1/2} \mathbb{E}\left[\Big(t|V_k|\sum_{\tau=1}^t\sum_{i\in V_k}\sigma^2_{\tau,i}(x_{\tau,i})\Big)^{1/2}\right] \\
&\leq  \beta_{t}^{1/2}\xi_{|V_k|}^{1/2} \sqrt{\frac{2t|V_k|\Psi_{t|V_k|}}{\log(1+\sigma_\epsilon^{-2})}} \leq \sqrt{\frac{2\xi_{|V_k|} t|V_k|\beta_t\Psi_{t|V_k|}}{\log(1+\sigma_{\epsilon}^{-2})}}
\end{align}

Lines (9) and (10) follow by the law of total expectation. Line (11) follows by noting that $\beta_\tau$ is increasing with $\tau.$ Recall that because $G_k$ is a subgraph of $G,$ for all agents $i\in V_k,$ $V_k\subset \{N(i)\cup i\}.$ Therefore, $\bar{D}_\tau$ is contained in $D_{\tau,i}$ for all agents $i\in V_k.$ In essence, $\bar{D}_\tau\subset D_{\tau,i}\implies \sigma_{\bar{D}_\tau}\geq \sigma_{D_{\tau,i}}.$ Line (12) applies this property. Line (13) follows from Lemma \ref{lem:xi} in Appendix \ref{appendix:additional_analysis}, which introduces $\xi_i$ to bound the information gain of the current set of evaluations. Line (14) follows from the definition of $\xi_i,$ which was stated in Equation \ref{eq:xi}. Line (15) is a consequence of the fact that $\xi_i$ is increasing with $i.$ Line (16) follows from application of the Cauchy-Schwarz inequality, and line (17) is a consequence of Lemma \ref{lem:MIG}, which bounds the sums of the posterior variances by the MIG term.

Recall that $\sum_{\tau=1}^t \sum_{i\in V_k}\mathbb{E}[f(x^*)- f(x_{\tau,i})]=S1+S2+S3.$ Therefore, $R_{AB,k}(t)\leq \frac{\sqrt{2}\pi^{3/2}}{12t|V_k|} +  \sqrt{\frac{2\xi_{|V_k|} \beta_t\Psi_{t|V_k|}}{t|V_k|\log(1+\sigma_{\epsilon}^{-2})}}.$ 

Equipped with a bound on $R_{AB,k}(t),$ we can revisit our expression for $R_{AB}(t).$

\begin{align*}
R_{AB}(t)&=\frac{1}{M}\sum_{k=1}^n |V_k|R_{AB,k}(t)\\
&\leq \frac{1}{M}\sum_{k=1}^n |V_k|
\left(\frac{C_1}{t|V_k|} +  \sqrt{\frac{C_2\xi_{|V_k|} \beta_t\Psi_{t|V_k|}}{t|V_k|}}\right),
\end{align*}

\noindent where $C_1=\frac{\sqrt{2}\pi^{3/2}}{12}$ and $C_2=\frac{2}{\log(1+\sigma_\epsilon^{-2})}.$ Thus, we have shown that $R_{AB}(t)\leq \frac{1}{M}\sum_{k=1}^n |V_k|
\left(\frac{C_1}{t|V_k|} +  \sqrt{\frac{C_2\xi_{|V_k|} \beta_t\Psi_{t|V_k|}}{t|V_k|}}\right),$ concluding our proof.
\end{proof}

By picking $n$ to be the clique cover number of the graph $G$, Theorem \ref{theorem:rab} yields the following corollary.

\begin{corollary}
    \label{corollary:avg_regret}
    Suppose $k(x,x') \leq 1$ for all $x, x'$. Let $\theta(G)$ and $\omega(G)$ denote the clique cover number and clique number of the graph $G$ respectively. Then, the Bayesian average regret after $t$ timesteps satisfies \newline
    $R_{AB}(t) \leq \frac{C_1 \theta(G)}{Mt} + \frac{\sqrt{\theta(G)} \sqrt{C_2 \xi_{\omega(G)} \beta_t \Psi_{t \omega(G)}}}{\sqrt{Mt}},$
    where $\beta_t, C_1$ and $C_2$ are as defined in Theorem \ref{theorem:rsb}.
\end{corollary}

\begin{proof}
The proof of Corollary \ref{corollary:avg_regret} follows from (i) applying Cauchy-Schwarz to bound the term $\sum_{k=1}^n \sqrt{|V_k|} \leq \sqrt{n}\sqrt{\sum_{k=1}^n |V_k|} = \sqrt{Mn}$, (ii) picking $n$ to be the clique cover number of $G$, $\theta(G)$, and (iii) the fact that for any clique $G_k = (V_k,E_k)$ in $G$, $|V_k| \leq \omega(G)$, since $\omega(G)$ denotes the clique number of $G$ (i.e. size of the largest clique in $G$).
\end{proof}

From Corollary \ref{corollary:avg_regret}, the average regret satisfies $R_{AB}(t) = \tilde{O}\left(\frac{\sqrt{\theta(G) \xi_{\omega(G)} \Psi_{t \omega(G)}}}{\sqrt{Mt}}\right)$ (recall $\omega(G)$ denotes the clique number of $G$). We note that the term $\Psi_{t\omega(G)}$ corresponds to the maximal mutual information gain from $t\omega(G)$ observations, and that this quantity depends only logarithmically on $t\omega(G)$ for standard kernels such as the squared exponential kernel. For more details, see Appendix \ref{appendix:Psi_bounds}. The term $\xi_{\omega(G)}$ is the price we pay for the absence of coordination within each of the subgraphs $G_k$ in $G$, and is a standard term that arises in multi-agent Bayesian optimization. By an appropriate initialization phase, this term can be reduced to $\tilde{O}(1),$ (see Appendix B.3 in \cite{ren2024minimizing}). Thus, compared to the sequential single-agent case with $t$ rounds which has average regret $\tilde{O}\left(\sqrt{\frac{1}{t}} \right)$ \cite{kakade}, our algorithm satisfies a regret of $\tilde{O}\left(\frac{\sqrt{\theta(G)}}{\sqrt{Mt}}\right)$, i.e. an improvement of $\sqrt{\frac{\theta(G)}{M}}$ (note this term is always smaller than 1). Correspondingly, the average regret is smaller for graphs with higher connectivity, whose clique cover number $\theta(G)$ is smaller. We next proceed to bound the Bayesian simple regret.

\begin{figure*}[ht!]
    \centering
    \begin{subfigure}[b]{0.225\textwidth}
    \includegraphics[width=1\textwidth]{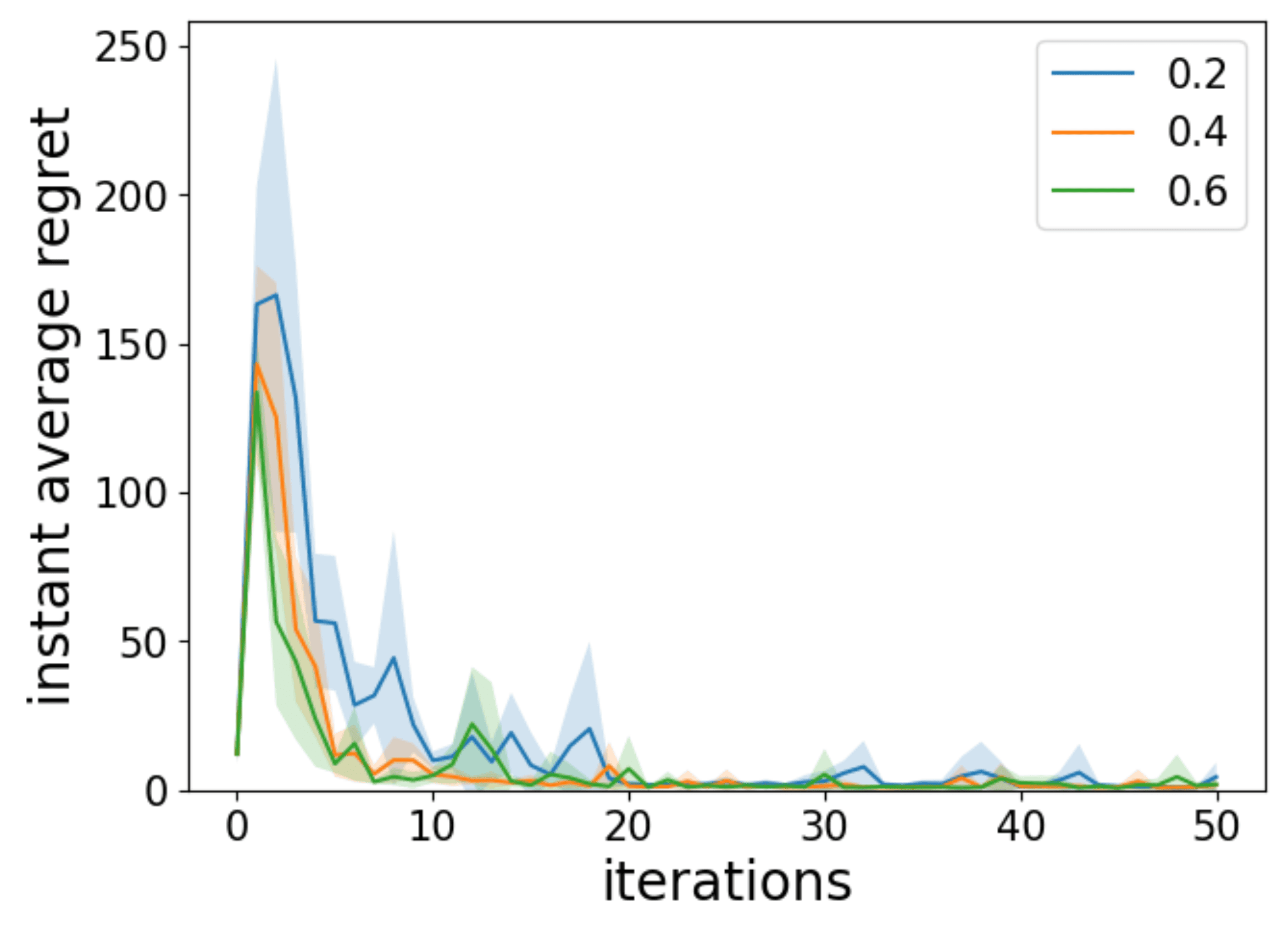}
        \label{fig:obj1a}
    \end{subfigure}
    \begin{subfigure}[b]{0.225\textwidth}
        \includegraphics[width=\textwidth]{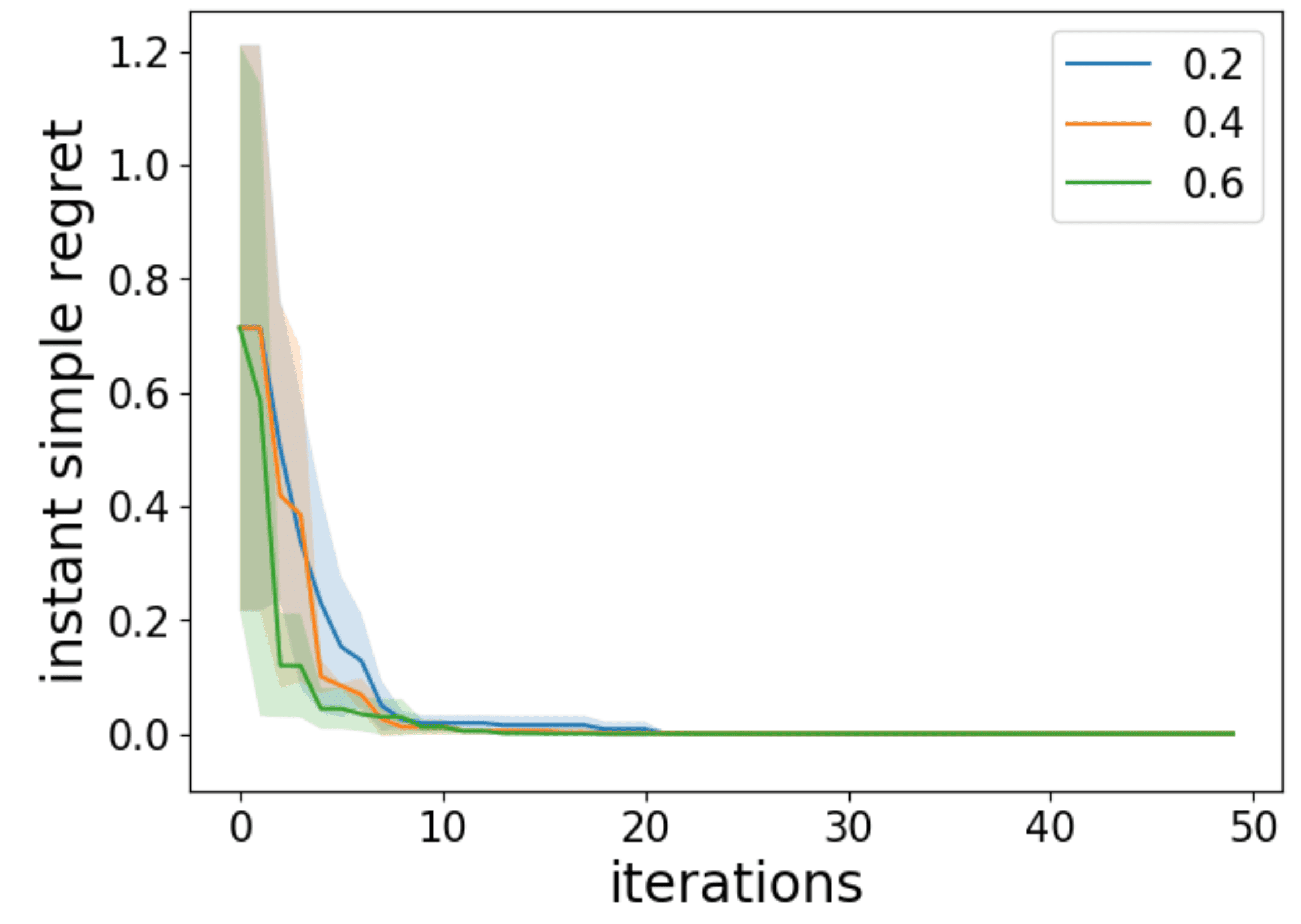}
        \label{fig:obj1b}
    \end{subfigure}
    \hfill
    \begin{subfigure}[b]{0.225\textwidth}
        \includegraphics[width=\textwidth]{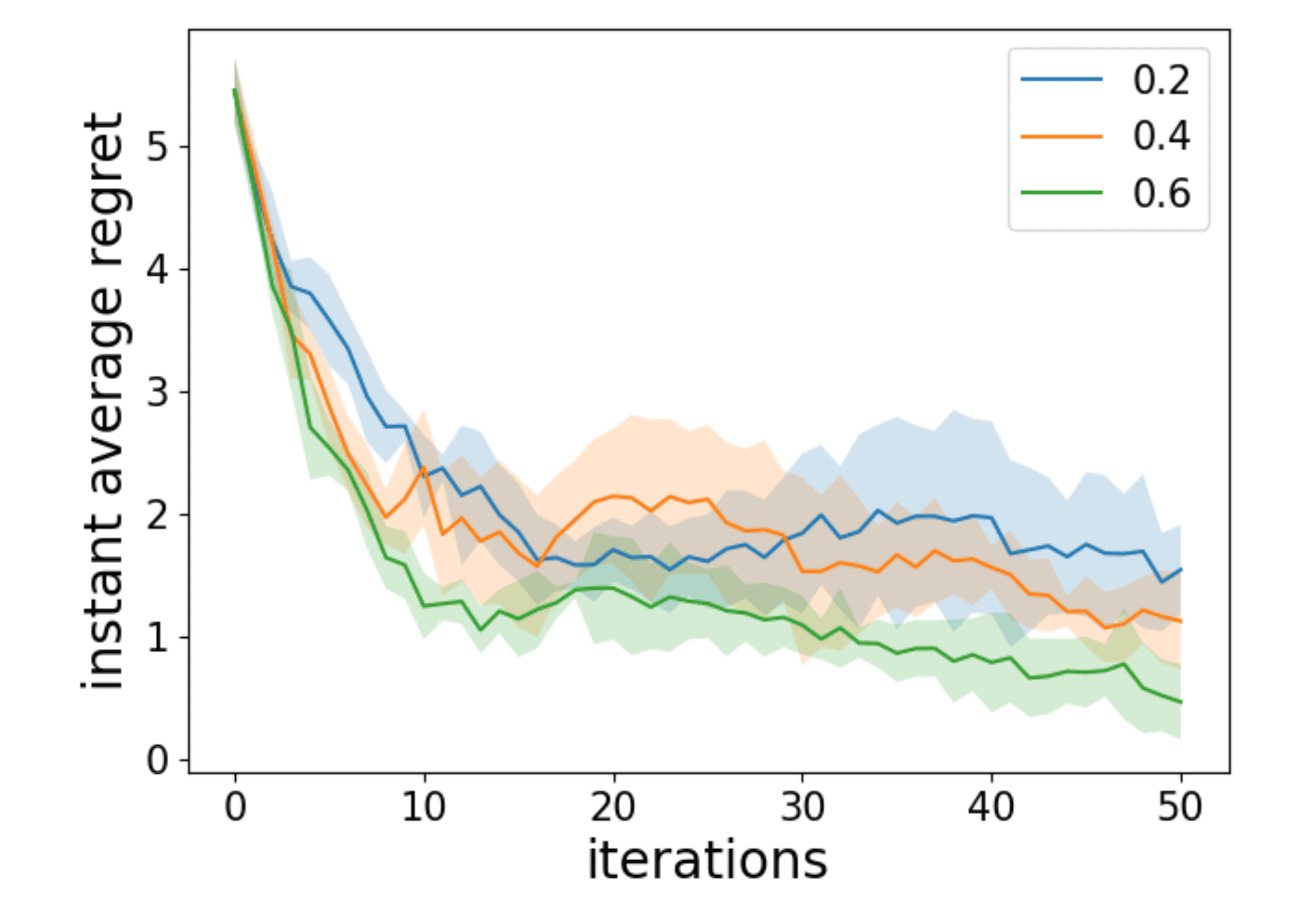}
        \label{fig:obj1c}
    \end{subfigure}
    \begin{subfigure}[b]{0.225\textwidth}
        \includegraphics[width=\textwidth]{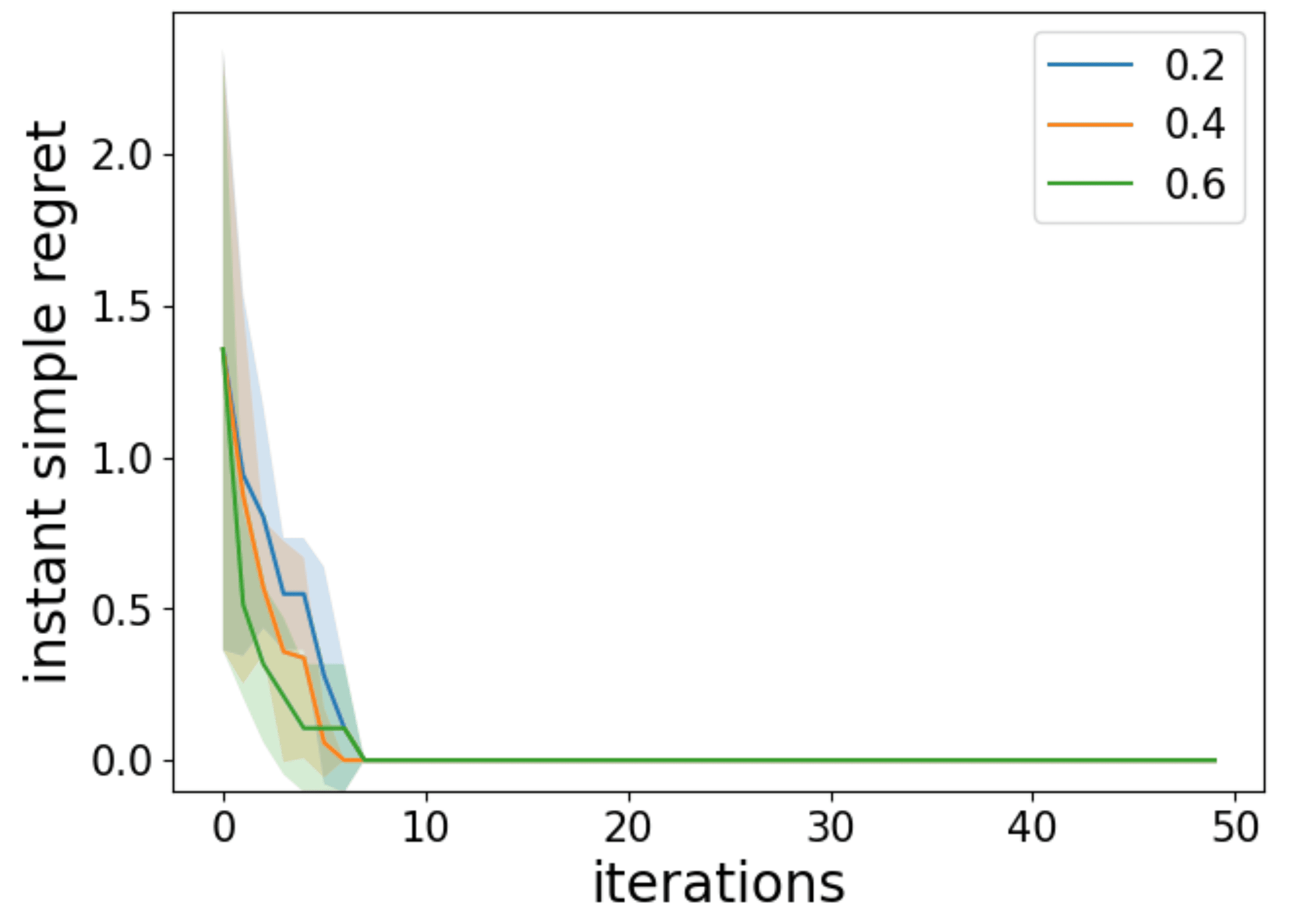}
        \label{fig:obj1d}
    \end{subfigure}

    \begin{subfigure}[b]{0.225\textwidth}
        \includegraphics[width=\textwidth]{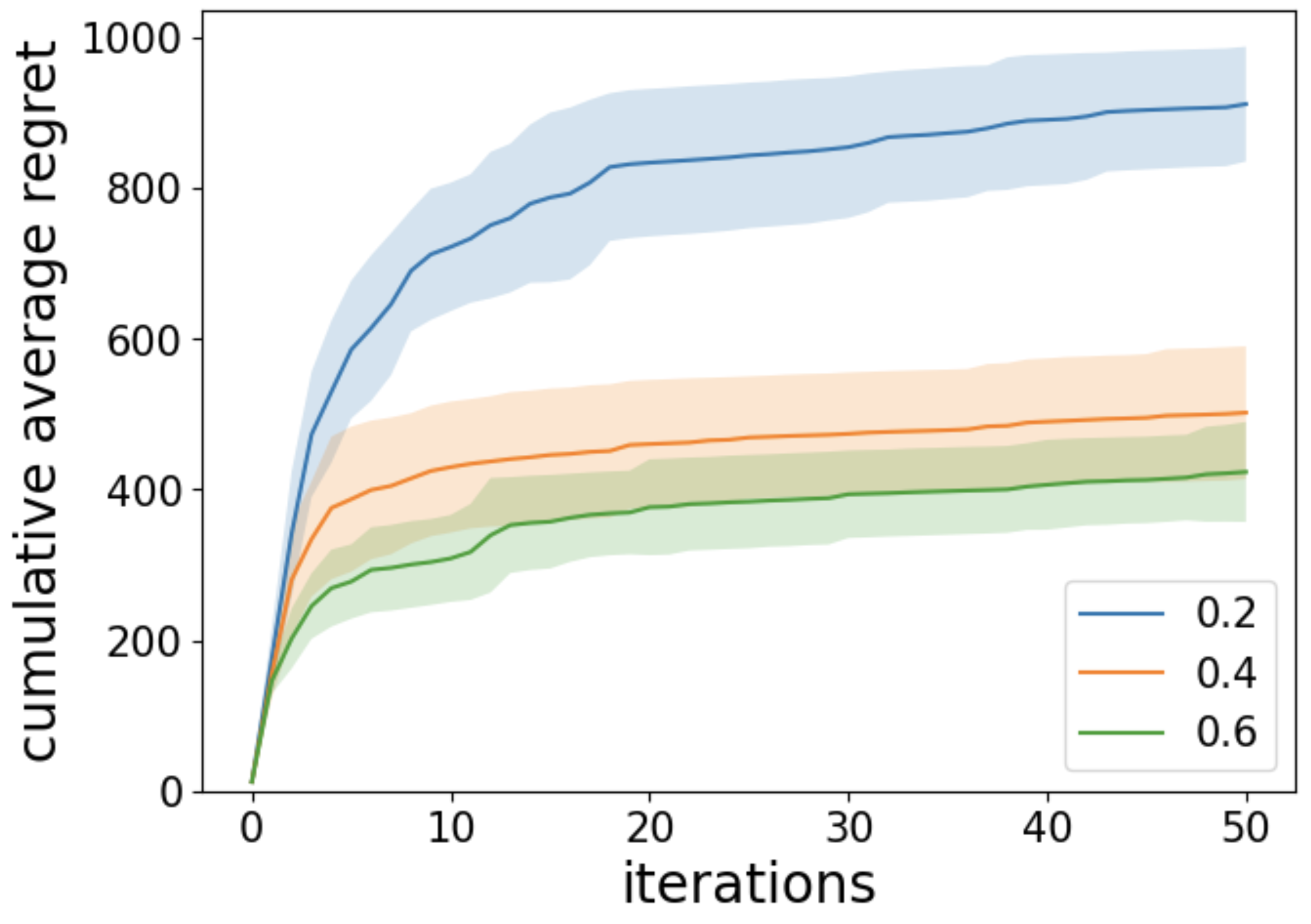}
        \label{fig:obj2a}
    \end{subfigure}
    \begin{subfigure}[b]{0.225\textwidth}
        \includegraphics[width=\textwidth]{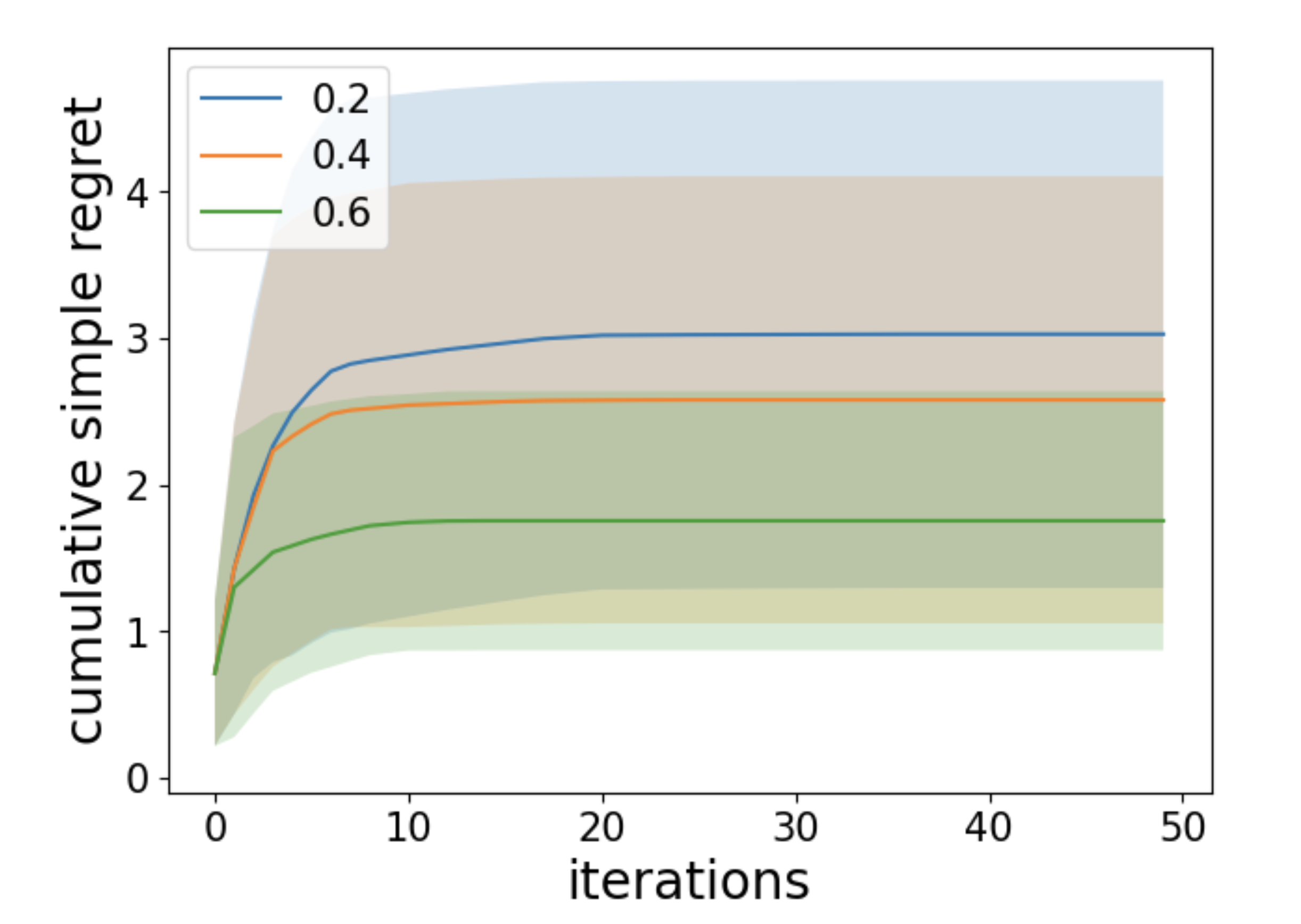}
        \label{fig:obj2b}
    \end{subfigure}
    \hfill
    \begin{subfigure}[b]{0.225\textwidth}
        \includegraphics[width=\textwidth]{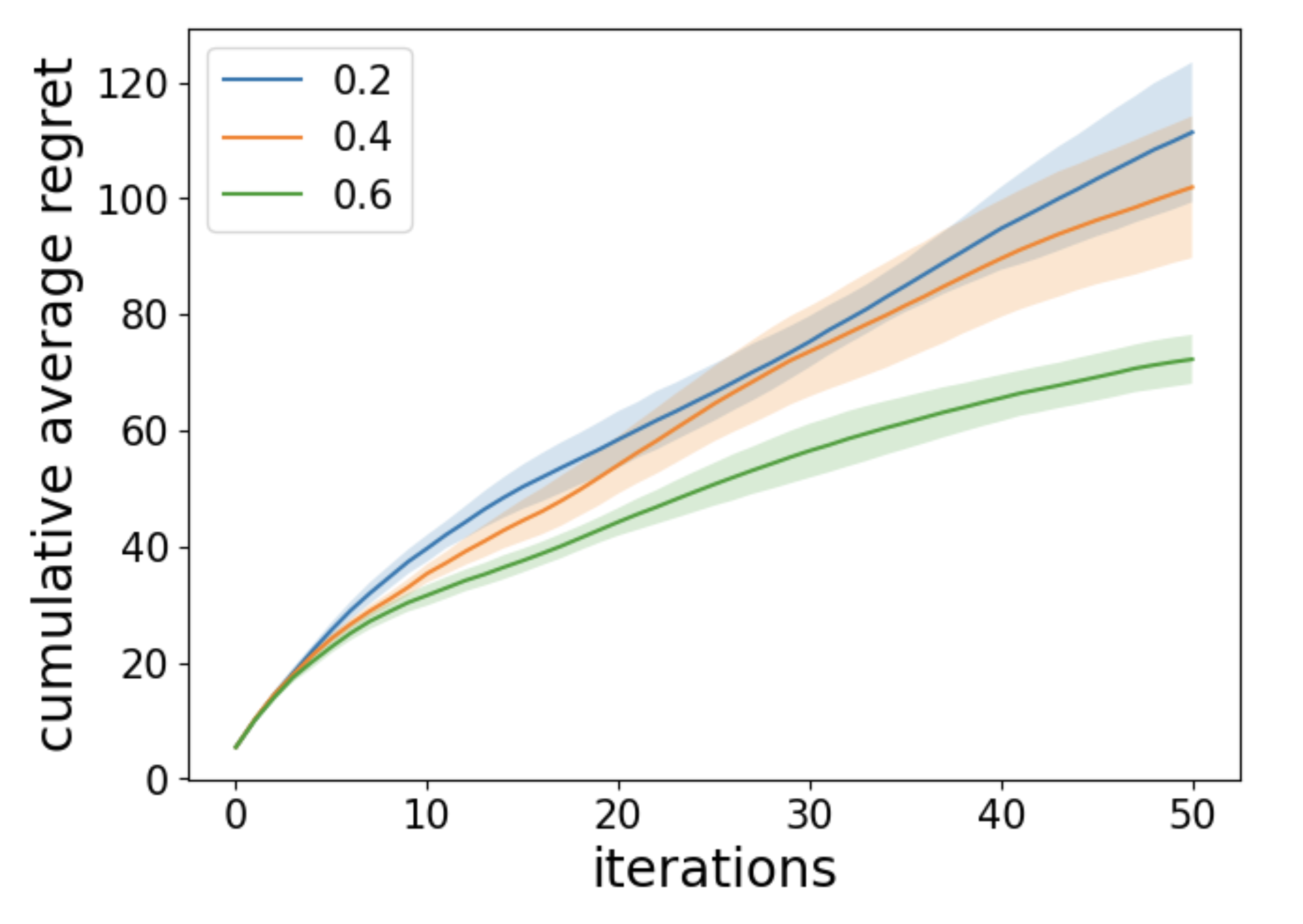}
        \label{fig:obj2c}
    \end{subfigure}
    \begin{subfigure}[b]{0.225\textwidth}
        \includegraphics[width=\textwidth]{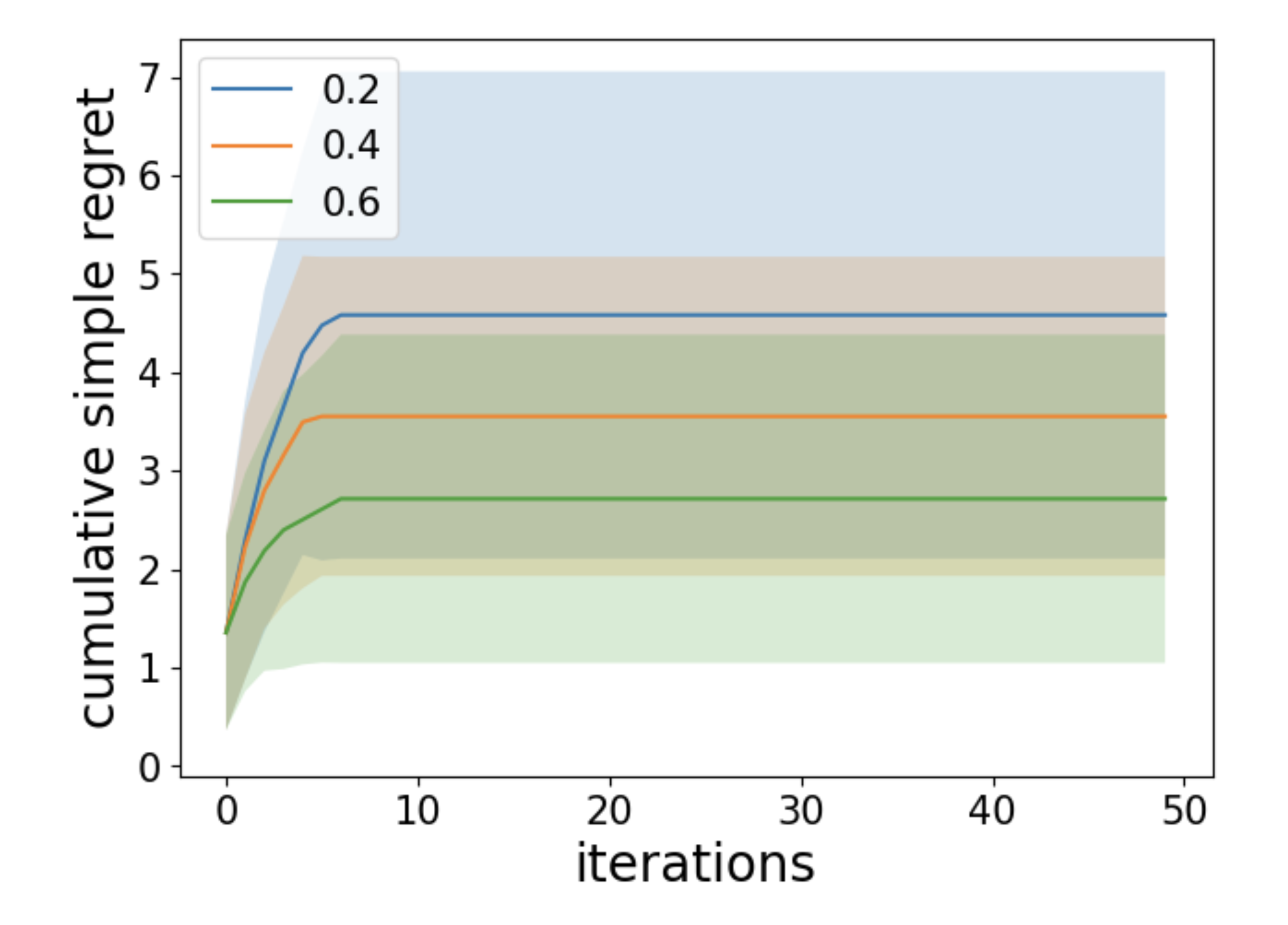}
        \label{fig:obj2d}
    \end{subfigure}
    \caption{Regret analysis of numerical simulations with $20$ agents on Erdős-Rényi random graphs with connectivity probability $0.2$ (blue), $0.4$ (orange), and $0.6$ (green), on Rosenbrock (left) and Ackley (right) objective functions.}
    \label{fig}
\end{figure*}

\begin{theorem}
\label{theorem:rsb}
Suppose $k(x,x')\leq 1$ for all $x,x'.$ Let $G_s=(V_s,E_s)$ be a complete subgraph of $G.$ Then the Bayesian simple regret after $t$ timesteps satisfies $R_{SB}(t)\leq \frac{C_1}{t|V_s|}+\sqrt{\frac{C_2\xi_{|V_s|}\beta_t\Psi_{t|V_s|}}{t|V_s|}},$ where $\beta_t=2\log(t^2|V_s||\mathcal{X}|),$ $C_1=\frac{\sqrt{2}\pi^{3/2}}{12},$ and $C_2=\frac{2}{\log(1+\sigma_\epsilon^{-2})}.$
\end{theorem}
\begin{proof}
 
Our proof begins by observing that the Bayesian simple regret is bounded by the Bayesian average regret of any subset of agents. 
Similar to the proof for Theorem \ref{theorem:rab}, we upper bound Bayesian average regret on a specific subset of agents by decomposing regret into three sums, each of which utilize the confidence function $U_{t,i}.$ The bounds follow notions from information theory using techniques from Kandasamy et al. \cite{parallelizedbo}, incorporating the information gain constants introduced in Equations \ref{eq:MIG} and \ref{eq:xi}.

Recall from equation \ref{eq:sb} that $R_{SB}(t)=\min_{i\in\{1,2,\ldots,M\},\tau \in \{1,2,\ldots,t\}}\mathbb{E}[f(x^*)- f(x_{\tau,i})].$ Define $U_{t,i}(x)=\mu_{D_{t,i}}(x)+\beta_t^{1/2}\sigma_{D_{t,i}}(x).$ Note that, since $R_{SB}(t)$ represents a minimum value of the expression $\mathbb{E}[f(x^*)- f(x_{\tau,i})],$ it is upper bounded by the average of this expression across any subset of agents. Thus, we may write
$$R_{SB}(t)\leq \frac{1}{t|V_s|}\sum_{\tau=1}^t\sum_{i\in V_s}\mathbb{E}[f(x^*)- f(x_{\tau,i})].$$ 

Thus, it suffices to upper bound $\frac{1}{t|V_s|}\sum_{\tau=1}^t\sum_{i\in V_s}\mathbb{E}[f(x^*)- f(x_{\tau,i})].$ Note that this expression is equal to the Bayesian average regret over $|V_s|$ agents on the complete subgraph $G_s=(V_s,E_s).$ By Theorem \ref{theorem:rab}, we can bound Bayesian average regret; thus, we may write $\frac{1}{t|V_s|}\sum_{\tau=1}^t\sum_{i\in V_s}\mathbb{E}[f(x^*)- f(x_{\tau,i})]\leq \frac{\sqrt{2}\pi^{3/2}}{12t|V_s|}+\sqrt{\frac{2\xi_{|V_s|}\beta_t\Psi_{t|V_s|}}{t|V_s|\log(1+\sigma_{\epsilon}^{-2})}}.$ Because $R_{SB}(t)\leq \frac{1}{t|V_s|}\sum_{\tau=1}^t\sum_{i\in V_s}\mathbb{E}[f(x^*)- f(x_{\tau,i})],$ $R_{SB}(t)$ inherits this bound. Therefore, we have shown $R_{SB}(t)\leq \frac{C_1}{t|V_s|}+\sqrt{\frac{C_2\xi_{|V_s|}\beta_t\Psi_{t|V_s|}}{t|V_s|}},$ where $\beta_t=2\log(t^2|V_s||\mathcal{X}|),$ $C_1=\frac{\sqrt{2}\pi^{3/2}}{12},$ and $C_2=\frac{2}{\log(1+\sigma_\epsilon^{-2})},$ concluding our proof.
\end{proof}

Picking $G_s$ to be the largest complete subgraph of the communication network $G$ then yields the following corollary. 

\begin{corollary}
\label{corollary:main}
Suppose $k(x,x')\leq 1$ for all $x,x'.$ Let $G_{max}=(V_{max},E_{max})$ be the largest complete subgraph of $G.$ Then the Bayesian simple regret after $t$ timesteps satisfies $R_{SB}(t)\leq \frac{C_1}{t|V_{max}|}+\sqrt{\frac{C_2\xi_{|V_{max}|}\beta_t\Psi_{t|V_{max}|}}{t|V_{max}|}},$ where $\beta_t=2\log(t^2|V_{max}||\mathcal{X}|),$ $C_1=\frac{\sqrt{2}\pi^{3/2}}{12},$ and $C_2=\frac{2}{\log(1+\sigma_\epsilon^{-2})}.$
\end{corollary}

From the above corollary and our discussion following Corollary \ref{corollary:avg_regret}, we see that $R_{SB}(t) = \tilde{O}\left(\sqrt{\frac{1} {t |V_{max}|}}\right)$.  Thus, compared to the sequential single-agent case with $t$ rounds which has simple regret $\tilde{O}\left(\sqrt{\frac{1}{t}} \right)$ \cite{kakade}, our algorithm satisfies a regret of $\tilde{O}\left(\sqrt{\frac{1}{t |V_{max}|}} \right)$, i.e. an improvement of $\sqrt{\frac{1}{|V_{max}|}}$, demonstrating the benefit of the network structure for the simple regret case as well. 

\section{Numerical Experiments}
\subsection{Simulation}

In the numerical implementation, performance was assessed utilizing the following regret metrics. We define the \emph{Instant average regret} $R_A$, and its sum, $\overline{R_A}$, as follows:

\begin{equation*}
R_A(t)=\frac{1}{M}\sum_{i=1}^M \left(f^*-f(x_{t,i})\right),\hspace{.25cm} \overline{R_A}(t)=\sum_{\tau=1}^t R_A(\tau).
\end{equation*}
\normalsize

We also define the \emph{Instant simple regret} $R_S$, and its sum $\overline{R_S}$, as follows:
\vspace{-.25cm}
\begin{equation*}
R_S(t)=f^*-\max_{\substack{i\in\{1,2,\ldots M\}, \\{\tau\in \{1,2,\ldots,t\}}}} f(x_{t,i}),\hspace{.25cm} \overline{R_S}(t)=\sum_{\tau=1}^t R_S(\tau),
\end{equation*}

where $f^*=\max_{x\in \mathcal{X}} f(x).$ Numerical results were constructed in a Python implementation built upon the BOTorch package \cite{botorch}. The code used to generate the simulation and corresponding Figure \ref{fig} is available at \href{https://github.com/sabzer/distributed-bo}{https://github.com/sabzer/distributed-bo}. The Gaussian processes utilized the Matérn kernel with parameter $\nu=\frac{5}{2}.$ The numerical simulations were run over $T=50$ timesteps. Simulations were run based on two test functions for the objective function: Ackley, which has many local maxima and one global minima in the origin, and Rosenbrock, which contains a large valley in which the global minima is situated. The equations of the aforementioned objective functions and their plots are available in Appendix \ref{app:obj}. Since we were solving a maximization problem, we multiplied the canonical definitions of these functions by $-1$ for the purpose of our simulation. For the communication networks in our simulations, we used Erdős-Rényi random graphs of $20$ agents with connectivities of $0.2,$ $0.4,$ and $0.6$ \cite{ER}. The connection probabilities are the probability that each edge from the complete graph of $20$ agents appears in the corresponding random graph.

\subsection{Discussion}

Our theoretical result bounds Bayesian average regret, $R_{AB}(t),$ and Bayesian simple regret, $R_{SB}(t),$ with the bound dependent on the structure of the communication network between agents. Our distributed Thompson sampling algorithm was able to achieve the extrema of the Ackley and Rosenbrock objective functions in numerical implementation, and thus is effective at the Bayesian optimization task. Our theoretical results suggests that the distributed Thompson sampling algorithm implementation favors highly connected communication graphs. This is apparent from a lower Bayesian average regret bound when the communication graph can be decomposed into a few large disjoint complete subgraphs, and a lower Bayesian simple regret bound when the largest complete subgraph of the communication graph has a larger number of agents. Our numerical results support this intuition, for in Figure \ref{fig}, we see better regret convergence for Erdős-Rényi graphs of higher connectivity. This result holds for both Ackley and Rosenbrock objective functions, and for both Instant simple and average regret.

\section{Conclusion}
In this paper, we proposed a distributed Thompson sampling algorithm to address the multi-agent Bayesian optimization problem under constrained communication. We develop bounds on Bayesian average regret and Bayesian simple regret for this approach, where the bound is dependent on properties of the largest complete subgraph of the graph encoding communication structure between agents. With our bound, we show that in connected multi-agent communication networks, both Bayesian average regret and Bayesian simple regret will converge faster with distributed Thompson sampling than in the sequential single-agent case, with the same number of rounds. Additionally, we demonstrate the efficacy of our algorithm with regret analysis on optimization test functions, illustrating faster convergence with well connected communication graphs. Future work will focus on developing a tighter regret bound, and further tailoring the distributed Thompson sampling algorithm towards the constrained communication case by leveraging the data communicated between agents.
\bibliography{arXiv_Final_v1} 

\begin{thebibliography}{10}

\bibitem{stochOpt}
J.~C. Spall, {\em Introduction to stochastic search and optimization:
  estimation, simulation, and control}.
\newblock John Wiley \& Sons, 2005.

\bibitem{hyperopt}
J.~Bergstra, D.~Yamins, and D.~Cox, ``Making a science of model search:
  Hyperparameter optimization in hundreds of dimensions for vision
  architectures,'' in {\em Proceedings of the 30th International Conference on
  Machine Learning}, vol.~28, pp.~115--123, PMLR, 2013.

\bibitem{smac3}
M.~Lindauer, K.~Eggensperger, M.~Feurer, A.~Biedenkapp, D.~Deng, C.~Benjamins,
  T.~Ruhkopf, R.~Sass, and F.~Hutter, ``Smac3: A versatile bayesian
  optimization package for hyperparameter optimization,'' {\em Journal of
  Machine Learning Research}, vol.~23, no.~54, pp.~1--9, 2022.

\bibitem{simOpt}
R.~Y. Rubinstein and D.~P. Kroese, {\em Simulation and the Monte Carlo method}.
\newblock John Wiley \& Sons, 2016.

\bibitem{chem}
R.-R. Griffiths and J.~M. Hernández-Lobato, ``Constrained bayesian
  optimization for automatic chemical design using variational autoencoders,''
  {\em Chem. Sci.}, vol.~11, pp.~577--586, 2020.

\bibitem{ea}
A.~Slowik and H.~Kwasnicka, ``Evolutionary algorithms and their applications to
  engineering problems,'' {\em Neural Computing and Applications}, vol.~32,
  pp.~12363--12379, 2020.

\bibitem{PSO}
F.~Marini and B.~Walczak, ``Particle swarm optimization (pso). a tutorial,''
  {\em Chemometrics and Intelligent Laboratory Systems}, vol.~149,
  pp.~153--165, 2015.

\bibitem{zeroB}
A.~Wibisono, M.~J. Wainwright, M.~Jordan, and J.~C. Duchi, ``Finite sample
  convergence rates of zero-order stochastic optimization methods,'' in {\em
  Advances in Neural Information Processing Systems}, 2012.

\bibitem{kushner}
H.~J. Kushner, ``{A New Method of Locating the Maximum Point of an Arbitrary
  Multipeak Curve in the Presence of Noise},'' {\em Journal of Basic
  Engineering}, vol.~86, no.~1, pp.~97--106, 1964.

\bibitem{mockus}
J.~Mo{\v{c}}kus, ``On bayesian methods for seeking the extremum,'' in {\em
  Optimization Techniques IFIP Technical Conference Novosibirsk, July 1--7,
  1974}, pp.~400--404, Springer Berlin Heidelberg, 1975.

\bibitem{gpbook}
C.~E. Rasmussen and C.~K.~I. Williams, {\em Gaussian Processes for Machine
  Learning}.
\newblock MIT Press, 2006.

\bibitem{frazier}
P.~I. Frazier, ``A tutorial on bayesian optimization,'' 2018.

\bibitem{zeroA}
Y.~Ha and S.~Shashaani, ``Iteration complexity and finite-time efficiency of
  adaptive sampling trust-region methods for stochastic derivative-free
  optimization,'' {\em IISE Transactions}, vol.~0, no.~0, pp.~1--15, 2024.

\bibitem{kakade}
N.~Srinivas, A.~Krause, S.~Kakade, and M.~Seeger, ``Gaussian process
  optimization in the bandit setting: no regret and experimental design,'' in
  {\em Proceedings of the 27th International Conference on International
  Conference on Machine Learning}, p.~1015–1022, 2010.

\bibitem{TS}
W.~R. Thompson, ``On the likelihood that one unknown probability exceeds
  another in view of the evidence of two samples,'' {\em Biometrika}, vol.~25,
  no.~3/4, pp.~285--294, 1933.

\bibitem{TSTut}
D.~Russo, B.~V. Roy, A.~Kazerouni, I.~Osband, and Z.~Wen, ``A tutorial on
  thompson sampling,'' 2020.

\bibitem{regret}
O.~Chapelle and L.~Li, ``An empirical evaluation of thompson sampling,'' in
  {\em Advances in Neural Information Processing Systems}, vol.~24, Curran
  Associates, Inc., 2011.

\bibitem{simpleregret}
S.~Vakili, N.~Bouziani, S.~Jalali, A.~Bernacchia, and D.-s. Shiu, ``Optimal
  order simple regret for gaussian process bandits,'' in {\em Proceedings of
  the 35th International Conference on Neural Information Processing Systems},
  2024.

\bibitem{regret2}
T.~Lai and H.~Robbins, ``Asymptotically efficient adaptive allocation rules,''
  {\em Advances in Applied Mathematics}, vol.~6, no.~1, pp.~4--22, 1985.

\bibitem{agrawal}
S.~Agrawal and N.~Goyal, ``Thompson sampling for contextual bandits with linear
  payoffs,'' in {\em Proceedings of the 30th International Conference on
  Machine Learning}, vol.~28, pp.~127--135, PMLR, 2013.

\bibitem{haitong}
H.~Ma, T.~Zhang, Y.~Wu, F.~P. Calmon, and N.~Li, ``Gaussian max-value entropy
  search for multi-agent bayesian optimization,'' 2023.

\bibitem{parallelizedbo}
K.~Kandasamy, A.~Krishnamurthy, J.~Schneider, and B.~Poczos, ``Parallelised
  bayesian optimisation via thompson sampling,'' in {\em Proceedings of the
  Twenty-First International Conference on Artificial Intelligence and
  Statistics}, vol.~84, pp.~133--142, PMLR, 2018.

\bibitem{moreBatch}
J.~Wu and P.~Frazier, ``The parallel knowledge gradient method for batch
  bayesian optimization,'' in {\em Advances in Neural Information Processing
  Systems}, vol.~29, Curran Associates, Inc., 2016.

\bibitem{ren2024minimizing}
Z.~Ren and N.~Li, ``Ts-rsr: a provably efficient algorithm for batch bayesian
  optimization,'' {\em arXiv preprint arXiv:2403.04764}, 2024.

\bibitem{distributed}
J.~M. Hern{\'a}ndez-Lobato, J.~Requeima, E.~O. Pyzer-Knapp, and
  A.~Aspuru-Guzik, ``Parallel and distributed thompson sampling for large-scale
  accelerated exploration of chemical space,'' in {\em Proceedings of the 34th
  International Conference on Machine Learning}, vol.~70, PMLR, 2017.

\bibitem{ss}
T.~Zhang, V.~Qin, Y.~Tang, and N.~Li, ``Distributed information-based source
  seeking,'' {\em IEEE Transactions on Robotics}, vol.~39, no.~6,
  pp.~4749--4767, 2023.

\bibitem{naomi}
K.~Nakamura, M.~Santos, and N.~E. Leonard, ``Decentralized learning with
  limited communications for multi-robot coverage of unknown spatial fields,''
  in {\em IEEE/RSJ International Conference on Intelligent Robots and Systems
  (IROS)}, pp.~9980--9986, 2022.

\bibitem{fullydbo}
J.~Garcia-Barcos and R.~Martinez-Cantin, ``Fully distributed bayesian
  optimization with stochastic policies,'' in {\em Proceedings of the
  Twenty-Eighth International Joint Conference on Artificial Intelligence,
  {IJCAI-19}}, International Joint Conferences on Artificial Intelligence
  Organization, 2019.

\bibitem{optTest}
M.~Jamil, X.-S. Yang, and H.-J. Zepernick, ``8 - test functions for global
  optimization: A comprehensive survey,'' in {\em Swarm Intelligence and
  Bio-Inspired Computation}, pp.~193--222, Elsevier, 2013.

\bibitem{russovanroy}
D.~Russo and B.~V. Roy, ``An information-theoretic analysis of thompson
  sampling,'' {\em Journal of Machine Learning Research}, vol.~17, no.~68,
  pp.~1--30, 2016.

\bibitem{botorch}
M.~Balandat, B.~Karrer, D.~Jiang, S.~Daulton, B.~Letham, A.~G. Wilson, and
  E.~Bakshy, ``Botorch: A framework for efficient monte-carlo bayesian
  optimization,'' in {\em Advances in Neural Information Processing Systems},
  vol.~33, pp.~21524--21538, Curran Associates, Inc., 2020.

\bibitem{ER}
P.~Erdős and A.~Rényi, ``On random graphs i,'' {\em Publ. math. debrecen},
  vol.~6, no.~290-297, p.~18, 1959.

\end{thebibliography}
\bibliographystyle{ieeetr}
\newpage
\onecolumn
\appendices

\section{Additional Analysis}
\label{appendix:additional_analysis}
\begin{lemma}\label{lem:ething} (\cite{parallelizedbo}\cite{russovanroy})
At step $j,$ for all $x\in \mathcal{A},$ $\mathbb{E}[\mathbb{I}\{f(x)>U_j(x)\}\cdot (f(x)-U_j(x))]\leq\frac{1}{\sqrt{2\pi}}e^{-\beta_j/2}.$
\end{lemma}
\begin{proof}
Since $f$ is a $GP,$ we know $f(x)|D_j\sim \mathcal{N}(\mu_{j}(x),\sigma_{j}^2(x)).$ Recall that $U_j(\cdot)=\mu_{j}(\cdot)+\beta^{1/2}_{j}\sigma_{j}(\cdot).$ Thus, we know that $f(x)-U_j(x)|D_j\sim \mathcal{N}(-\beta^{1/2}_{j}\sigma_{j}(x),\sigma_{j}^2(x)).$ For a normal distribution $Z\sim \mathcal{N}(\mu,\sigma^2),$ with $\mu\leq 0,$ we have that $\mathbb{E}[Z\mathbb{I}(Z>0)]\leq\frac{\sigma}{\sqrt{2\pi}}e^{-\mu^2/(2\sigma^2)}.$ Thus, we can apply this fact to our setting, yielding: 
\begin{align*}
\mathbb{E}[\mathbb{I}\{f(x)>U_j(x)\}\cdot (f(x)-U_j(x))]&\leq\frac{\sigma_j(x)}{\sqrt{2\pi}} e^{-\beta_j/2} &\text{(Aforementioned property of } \mathcal{N}\text{)}\\
&\leq \frac{1}{\sqrt{2\pi}}e^{-\beta_j/2} &\text{(} \sigma_j(x)\leq \kappa(x,x)\leq 1\text{)}
\end{align*}
\end{proof}

 \begin{lemma}\label{lem:5.3}\cite{kakade}
The information gain for selected points can be expressed in terms of the predictive variances. if $f_{[n]}=(f(x_n))\in \mathbb{R}^n:$ $$I(y_{[n]};f_{[n]})=\frac12\sum_{j=1}^n \log(1+\sigma_\epsilon^{-2}\sigma_{j-1}^2(x_n))$$
\end{lemma}
\begin{proof}
\begin{align*}
I(y_{[n]};f_{[n]})&=H(y_{[n]})+H(y_{[n]}|f_{[n]})&\text{(Definition of information)}\\
&=H(y_{[n]})-\frac{1}{2}\log|2\pi e \sigma_\epsilon^2\mathbf{I}|&\text{(Gaussian entropy: } H(N(\mu,\Sigma))=\frac{1}{2}\log |2\pi e\Sigma|\text{)}
\end{align*}
Now let's develop an expression for $H(y_{[n]}):$
\begin{align*}
H(y_{[n]})&=H(y_{[n-1]}) + H(y_{n}|y_{[n-1]}) &\text{(Entropy chain rule: } H(A,B)=H(A)+H(B|A)\text{)}\\\\
&=H(y_{[n-1]})+\frac12\log(2\pi e (\sigma_\epsilon^2+\sigma^2_{n-1}(x_n))/2 &\text{(} y_n|y_{[n-1]}=f(x_n)+\epsilon_n, y_n|y_{[n-1]}\sim N(\mu_{n-1}(x_n),\sigma^2_{n-1}(x_n) +\sigma_\epsilon^2{)}
\end{align*}
We now have developed a recursive relation for $H(y_{[n]});$ we can inductively show that $H(y_{[n]})=\sum_{j=1}^n\frac12\log(2\pi e (\sigma_\epsilon^2+\sigma^2_{j-1}(x_n))/2.$ Utilizing this expression, we can return to our mutual information expression:
\begin{align*}
I(y_{[n]};f_{[n]})&=H(y_{[n]})-\frac{1}{2}\log|2\pi e \sigma_\epsilon^2\mathbf{I}|&\text{(Previously shown)}\\
&=\sum_{j=1}^n\frac12\log(2\pi e (\sigma_\epsilon^2+\sigma^2_{j-1}(x_n))/2-\frac{1}{2}\log|2\pi e \sigma_\epsilon^2\mathbf{I}|&\text{(Substituting developed } H(y_{[n]}\text{ expression)}\\
&=\frac12\sum_{j=1}^n \log(1+\sigma_{j-1}^2(x_n)\sigma_\epsilon^{-2}) &\text{(Simplifying log subtraction)}
\end{align*}
\end{proof}

\begin{lemma}\label{lem:MIG}\cite{parallelizedbo}
Let $f\sim GP(0,\kappa),f:\mathcal{X}\to\mathbb{R}$ and each time we query $x\in\mathbb{X}$ we observe $y=f(x)+\varepsilon,$ where $\varepsilon\sim \mathcal{N}(0,\sigma_n^2).$ Let $\{x_1,\ldots,x_t\}$ be an arbitrary set of $t$ evaluations to $f$ where $x_j\in \mathcal{X}$ for all $j.$ Let $\sigma_{j}^2$
 denote the posterior variance conditioned on the first $j$ of these queries, $\{x_1,\ldots,x_{j}\}.$ Then, $\sum_{j=1}^n \sigma_{j-1}^2(x_j)\leq \frac{2\Psi_n}{\log(1+\sigma_\epsilon^{-2})}.$
 \end{lemma}
 \begin{proof}
\begin{align*}
\Psi_n&\geq I(y_{[n]};f_{[n]})&\text{(Definition of } \Psi_n \text{)}\\
&\geq \frac12 \sum_{j=1}^n \log(1+\sigma_\epsilon^{-2}\sigma_{j-1}^2(x_j))&\text{(Lemma \ref{lem:5.3})}
\end{align*}
Note that the function $\frac{x}{1+\log(x)}$ increases with $x.$ Also note that $0\leq \sigma_\epsilon^2,\sigma_j^2(x_j)\leq 1$ by assumption, and thus $\sigma_\epsilon^{-2}\sigma_{j}^2(x_j)\leq \sigma_\epsilon^{-2}.$ Therefore, $$\frac{\sigma_\epsilon^{-2}\sigma_{j-1}^2(x_j)}{\log(1+\sigma_\epsilon^{-2}\sigma_{j-1}^2(x_j))}\leq \frac{\sigma_\epsilon^{-2}}{\log(1+\sigma_\epsilon^{-2})}\implies \sigma_{j-1}^2(x_j)\log(1+\sigma_{\epsilon}^{-2})\leq \log(1+\sigma_\epsilon^{-2}\sigma_{j-1}^2(x_j)).$$
Now we can revisit our earlier $\Psi_n$ expression:
\begin{align*}
\Psi_n&\geq \frac12 \sum_{j=1}^n \log(1+\sigma_\epsilon^{-2}\sigma_{j-1}^2(x_j))&\text{(Lemma \ref{lem:5.3})}\\
&\geq \frac12\log(1+\sigma_{\epsilon}^{-2})\sum_{j=1}^n \sigma_{j-1}^2(x_j) &\text{(Previously shown)}
\end{align*}
With algebraic manipulation, the last expression is equivalent to $\sum_{j=1}^n \sigma_{j-1}^2(x_j)\leq \frac{2\Psi_n}{\log(1+\sigma_\epsilon^{-2})},$ our desired statement.
 \end{proof}
 \begin{lemma}\label{lem:xi}
 Let $f\sim GP(0,\kappa),$ and let $A,B$ be finite subsets of $\mathcal{X}.$ Let $y_A\in\mathbb{R}^{|A|}$ and $y_B\in \mathbb{R}^{|B|}$ denote the observations when we evaluate $f$ at $A$ and $B.$ Let $\sigma_A,\sigma_{A\cup B}:\mathcal{X}\to\mathbb{R}$ denote the posterior standard deviation of the $GP$ when conditioned on $A$ and $A\cup B,$ respectively. Then $$\forall x\in\mathcal{X},\hspace{1cm}\frac{\sigma_A(x)}{\sigma_{A\cup B}(x)}= exp(I(f;y_B|y_A))$$
 \end{lemma}
 \begin{proof}
 \begin{align*}
 I(f;y_B|y_A)&=H(f|y_A)-H(f|y_{A\cup B})&\text{(Mutual information with conditional entropy)}\\
 &=\frac12\log |2\pi e\sigma_A^2|-\frac12\log |2\pi e\sigma_{A\cup B}^2|&\text{(Entropy for GP)}\\
 &=\log \left(\frac{\sigma_A}{\sigma_{A\cup B}}\right)&\text{}
 \end{align*}
 Thus,  $\frac{\sigma_A}{\sigma_{A\cup B}}=exp(I(f;y_B|y_A))$
 \end{proof}

\section{Bounds for the information gain quantity $\Psi_{\tau}$ for different kernels}\label{appendix:Psi_bounds}

We note that following a known result in \cite{kakade}, $\Psi_{\tau}$ in fact satisfies sublinear growth for three well-known classes of kernels, namely the linear, exponential and Matern kernels.

\begin{lemma}[cf. Theorem 5 in \cite{kakade}]
    \label{lemma:MIG_kernel_bounds}
For any $\tau > 0$, the maximal information gain $\Psi_{\tau}$ can be bounded as follows for the following kernels.
    \begin{enumerate}
        \item (Linear kernel): If $k(x,x')=x^\top x'$, then $$\Psi_{\tau} = O(d \log (\tau)).$$
        \item (Squared exponential kernel): If $k(x,x')=\exp(-\|x-x'\|^2/2)$, then
        $$\Psi_{\tau} = O(\left(\log(\tau)\right)^{d+1}).$$
        \item (Matern kernel with $\nu > 1$): If $k(x,x')= \frac{1}{\Gamma(\nu)2^{\nu-1}} \left( \frac{\sqrt{2\nu}}{d} \|x-x'\| \right)^v K_v\left( \frac{\sqrt{2v}}{d} \|x-x'\| \right)$,
where  $K_v(\cdot)$ is a modified Bessel function, and $\Gamma(\cdot)$ denotes the gamma function, then
        $$\Psi_{\tau} = O( (\tau)^{\frac{d(d+1)}{2\nu + d(d+1)}}\log(\tau))$$
    \end{enumerate}
\end{lemma}

\section{Objective Functions}\label{app:obj}
The test objective functions used in our simulations were the Rosenbrock and Ackley functions. The equations of those functions are as follows:

\textit{Rosenbrock}: \begin{equation}f(x, y) = (1 - x)^2 + 100(y - x^2)^2\end{equation}

\textit{Ackley}: \begin{equation}f(x, y) = -20 \exp\left( -0.2 \sqrt{\frac{x^2 + y^2}{2}} \right) - \exp\left( \frac{1}{2} (\cos(2 \pi x) + \cos(2 \pi y)) \right) + 20 + \exp(1)
\end{equation}

\begin{figure}[H]
    \centering
    \begin{subfigure}{0.45\textwidth}
        \centering
        \includegraphics[width=\linewidth]{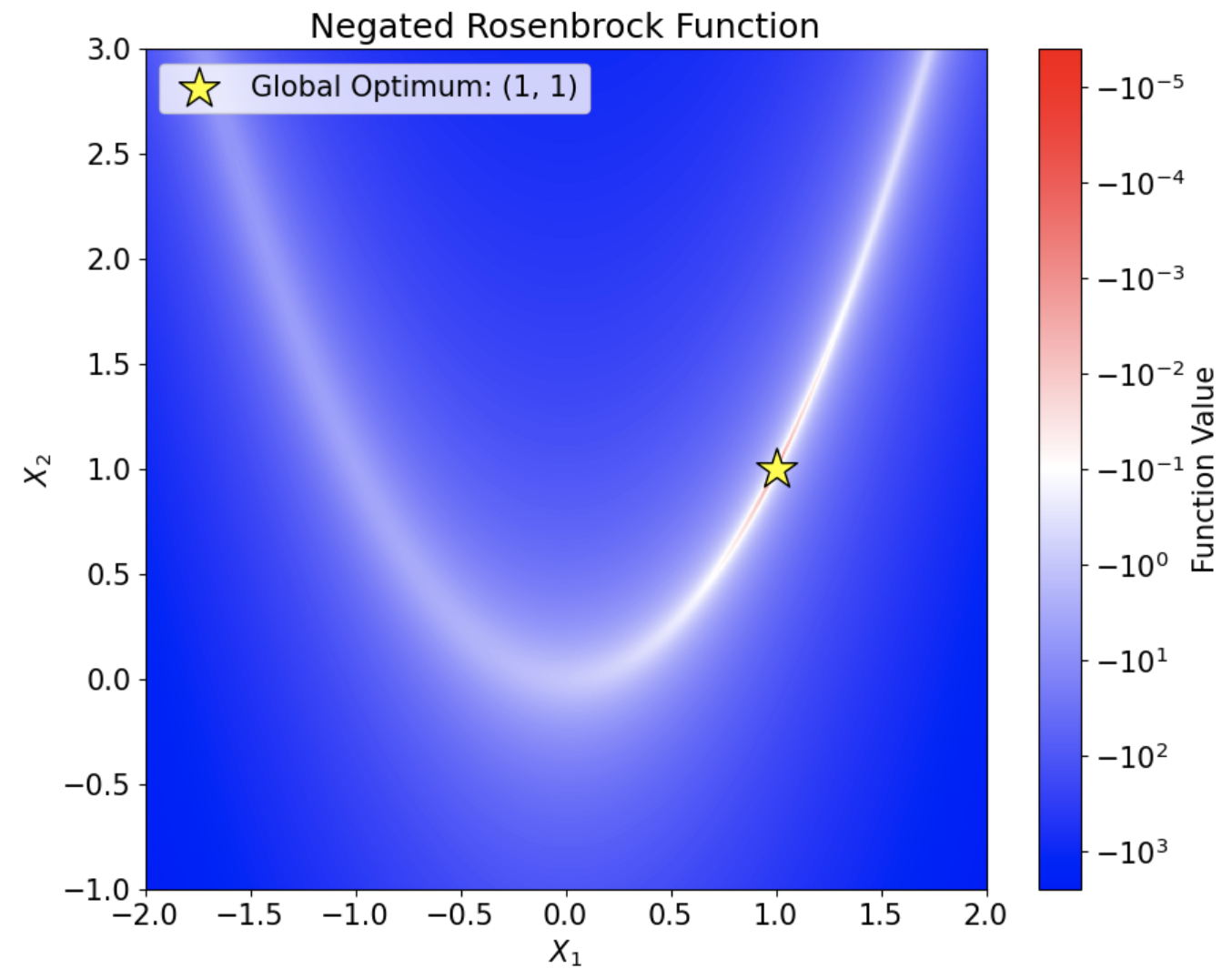}
        \caption{Plot of Rosenbrock Function.}
    \end{subfigure}
    \hfill
    \begin{subfigure}{0.45\textwidth}
        \centering
        \includegraphics[width=\linewidth]{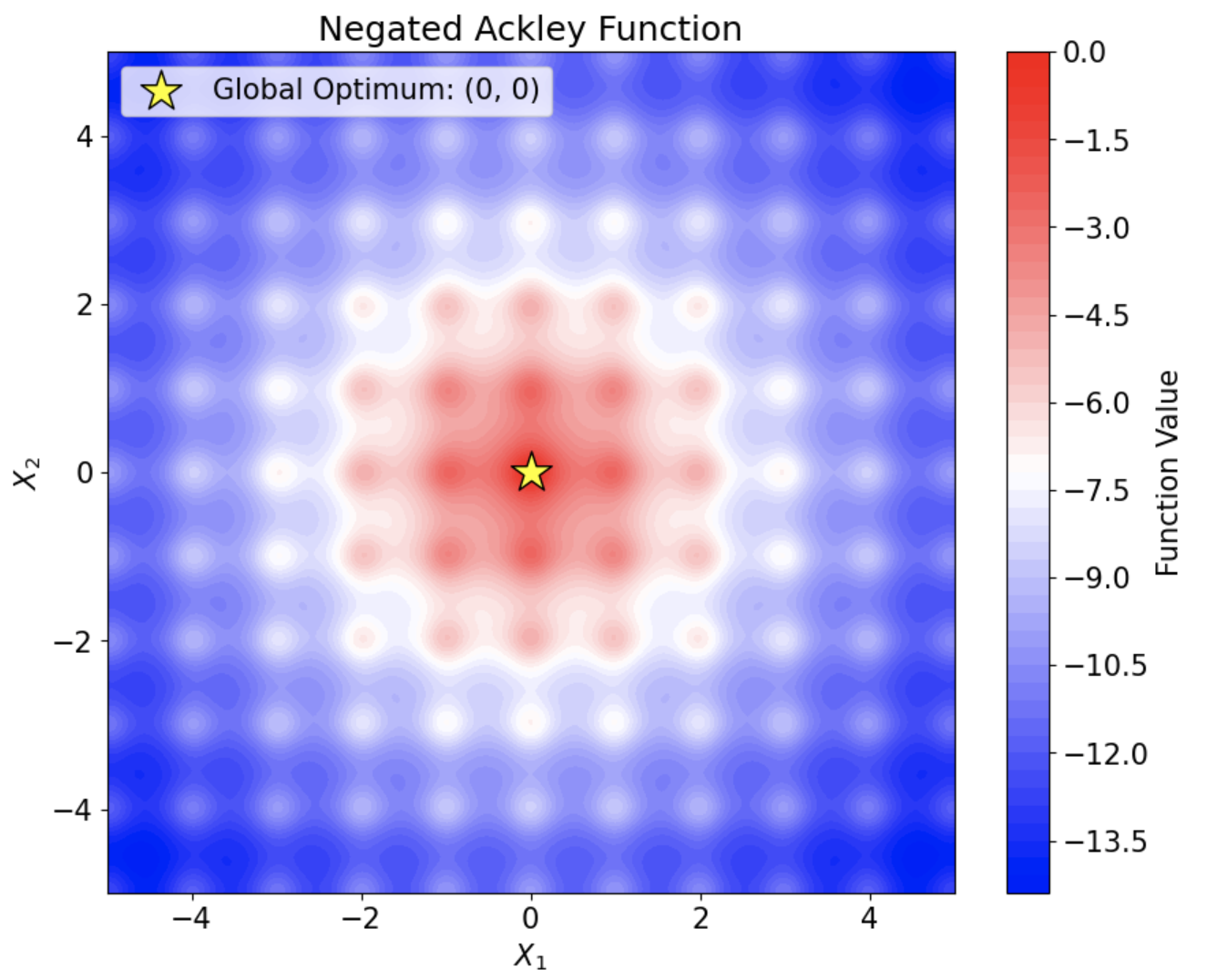}
        \caption{Plot of Ackley Function.}
    \end{subfigure}
    \caption{Plots of test objective functions.}
\end{figure}

\end{document}